\documentclass[10pt,twocolumn]{IEEEtran}
\usepackage[latin1]{inputenc}
\usepackage{amsmath}
\usepackage{amsfonts}
\usepackage{float}
\usepackage{amssymb,amsmath}
\usepackage{tabularx}
\usepackage{graphicx}
\ifCLASSINFOpdf
\else
\fi

\newtheorem{theorem}{Theorem}
\newtheorem{corr}{Corollary}

\begin{document}
\title{Noise Tolerance Under Risk Minimization}
\author{Naresh~Manwani,
        P.~S.~Sastry,~\IEEEmembership{Senior Member,~IEEE}% <-this % stops a space
\thanks{Naresh Manwani and P. S. Sastry are with the Department
of Electrical Engineering, Indian Institute of Science, Bangalore,
560012 India e-mail: (naresh,sastry@ee.iisc.ernet.in).}
} 

\maketitle
\begin{abstract}
In this paper we explore noise tolerant learning of classifiers. We formulate the problem as follows.
We assume that there is an ${\bf unobservable}$ training set which is noise-free. The actual training set given to the learning algorithm
is obtained from this ideal data set by corrupting the class label of each example. The probability that the class label
of an example is corrupted is a function
of the feature vector of the example. This would account for most  kinds of noisy data one encounters in practice.
We say that a learning method is noise tolerant if the classifiers learnt with noise-free data 
and with noisy data, both have the same classification accuracy on the noise-free data.
In this paper we analyze the noise tolerance properties of risk minimization (under 
different loss functions).
We show that risk minimization under 0-1 loss function has impressive noise tolerance properties and that under
squared error loss is tolerant only to uniform noise; 
risk minimization under other loss functions is not noise tolerant. 
We conclude the paper with some discussion on implications of these theoretical results. 
\end{abstract}

\IEEEpeerreviewmaketitle
\def \bw {\tilde{\mathbf{w}}}
\def \by {\tilde{\mathbf{y}}}
\def \bx {\tilde{\mathbf{x}}}
\def \bz {\tilde{\mathbf{Z}}}
\def \bX {\tilde{\mathbf{X}}}
\def \xx {\mathbf{x}}
\def \yy {\mathbf{y}}
\def \ee {\mathbf{e}}
\def \ww {\mathbf{w}}
\def \amax {\operatorname{argmax}}
\def \amin {\operatorname{argmin}}
\def \muu  {\mbox{\boldmath $\mu$}}

\section{Introduction}
In most situations of learning a classifier, one has to contend with noisy examples. Essentially, when training examples are noisy, the class
labels of examples as provided in the training set may not be `correct'. Such noise can come through many sources. 
If the class conditional densities overlap, then same feature vector can come from different classes
with different probabilities and this can be one source of noise. 
In addition, in many applications (e.g, document classification etc.), training examples are obtained through manual 
classification and there will be inevitable human errors and biases.
Noise in training data can come about by errors in 
feature measurements also. Errors in feature values would imply that the observed feature vector is at 
a different point in the feature space though the label remains the same and hence it can also be looked 
at as a noise corruption of the class label. It is always desirable to have classifier design strategies 
that are robust to noise in training data. 

A popular methodology in classifier learning is (empirical) risk minimization \cite{Duda,Bishop2006}.  
Here, one chooses a convenient loss function and the goal of learning is to find
a classifier that minimizes risk which is expectation of the loss. The expectation is with respect to the 
underlying probability distribution over the
feature space. In case of noisy samples, this
expectation would include averaging with respect to noise also.

In this paper, we study noise tolerance properties of risk minimization under different loss functions such as 0-1 loss, squared error loss, 
exponential loss, hinge loss etc. We consider what we call non-uniform noise where the probability of the class label for an example 
being incorrect, is a function of the feature vector of the example. This is a very general noise model and can account for most cases 
of noisy datasets. 
We say that risk minimization (with a loss function) is noise tolerant if the  minimizers under the noise-free 
and noisy cases have the same probability of mis-classification on noise-free datasets. We present some analysis to 
characterize noise tolerance of risk minimization with different 
loss functions. 

As we show here, the 0-1 loss function has very good noise tolerance properties, In general, risk minimization under 0-1 loss 
is desirable because it achieves least probability of mis-classification. However, the optimization problem here is computationally 
hard. 
To overcome this, many of the classifier learning strategies 
use some convex surrogates of the 0-1 loss function (e.g., hinge loss, squared error loss etc.). 
The convexity of the resulting optimization problems makes these approaches computationally efficient.
There have been statistical analyses of such methods so that one can bound risk under 0-1 loss, of the classifier obtained as a 
minimizer of risk under some other convex loss \cite{bartlett}. The analysis we present here is completely different because the 
objective is to understand noise tolerance properties of risk minimization. 
Here we are interested in comparing minimizers of risk under the same loss function but under different noise conditions.

The rest of the paper is organized as follows. In section 2 we discuss the concept of noise tolerant learning of classifiers. 
In section 3, we present our results regarding
noise tolerance of different loss functions. 
We  present a few simulation results to support our analysis 
in section~4 and conclude in section 5.

\section{Noise Tolerant Learning}
When we have to learn with noisy data where class labels may be corrupted, we want approaches that are robust to label noise.
Most of the standard classifiers (e.g. support vector machine, adaboost etc.) perform well only under noise-free training data; 
when there is label noise, they
tend to over-fit.

There are many approaches to tackle label noise in training data. Outliers detection
\cite{Brodley1999}, restoration of clean labels for noisy points \cite{Finedual99} and restricting the effects of noisy points
on the classifier \cite{Xu06,Tao2005} are some of the well known tricks to tackle the label noise.
However all these are mostly heuristic and also need extra computation.
Many of them also assume uniform noise and sometimes assume knowledge of noise variance.

A different approach would be to look for methods that are inherently noise tolerant. That is, the algorithm will handle the noisy data the same
way that it would handle noise-free data. However due to some property of the algorithm, its output would be same whether the input is noise
free or noisy data. 

Noise tolerant learning using statistical queries \cite{Kearns98} is one such approach. The algorithm learns 
 by using some statistical quantities computed
from the examples. That is the reason for its noise tolerance properties.
However, the approach is mostly limited to binary features. Also, the appropriate statistical quantities to
be computed depends on the type of noise and the type of classifier being learned. 

In this paper, we investigate the noise tolerance properties of the general risk minimization
strategy.
We formulate
our concept of noise tolerance as explained below. For simplicity,  we consider only the 
 two class classification problem.  

We assume that there exists an ideal noise-free sample which is {\em unobservable} but where the class label given to each example is correct. 
We represent this ideal sample by 
$\{({\bf x}_i, y_{{\bf x}_i}),i=1\ldots N\}$, where $\xx_i \in \Re^d,y_{{\bf x}_i}\in\{-1,1\},\forall i$

The actual training data given to the learning
algorithms is obtained by corrupting these (ideal) noise-free examples by changing the class label on each example. 
The actual training data set would be 
$\{({\bf x}_i, \hat{y}_{{\bf x}_i}),i=1\ldots N\}$,
where $\hat{y}_{{\bf x}_i}=y_{{\bf x}_i}$ with probability $(1-\eta_{{\bf x}_i})$ and is $\hat{y}_{{\bf x}_i}=-y_{{\bf x}_i}$ with probability
$\eta_{{\bf x}_i},\;\forall i$. 
If $\eta_{\xx_i}=\eta,\;\forall {\xx_i}$, then we say that the noise is {\em uniform}.
Otherwise, we say noise is {\em non-uniform}. 

We note here that under non-uniform classification noise,
the probability of the class label being wrong can be different for different examples. 
{\em We assume, throughout this paper, that $\eta_{\bf x} < 0.5$, $\forall {\bf x}$, which is  
reasonable.}

As a notation, we assume that the risk is defined over class of functions, $f$, that map feature space to real numbers. 
This allows us to treat all loss functions through a single notation. We call any such $f$ a classifier and the class label assigned  
by it to a feature vector $\xx$ would be $\mbox{sign}(f(\xx))$. 

Let $L(\cdot, \cdot)$ be a specific loss function. 
For any classifier $f$, the risk under no-noise case is
\begin{eqnarray}
R(f) = E\left[ L(f({\bf x}), y_{\bf x})\right] \nonumber
\end{eqnarray}
where $L(.,.)$ is the loss function. The expectation here is with respect to the underlying distribution of the feature vector ${\bf x}$. 
Let $f^*$ be the minimizer of $R(f)$.

Under the noisy case, the risk of any classifier $f$ is,
\begin{equation}
R^{\eta}(f) = E\left[ L(f({\bf x}), \hat{y}_{\bf x})\right]\nonumber
\end{equation}
Note that $\hat{y}_{\bf x}$ has additional randomness due to noise corruption of labels and 
the expectation includes averaging with respect to that also.
To emphasize this, we use the notation $R^{\eta}$ to denote risk under noisy case.
Let $f_{\eta}^*$ be the minimizer of $R^{\eta}$.

{\bf Definition~1}: Risk minimization under loss function $L$, is said to be {\em noise-tolerant} if 
$P[\mbox{sign}(f^*(\xx))=y_{\xx}]=P[\mbox{sign}(f^*_{\eta}(\xx))=y_{\xx}]$, 
where the probability is w.r.t. the underlying distribution of $(\xx, \: y_{\xx})$.

That is, the general learning strategy of risk minimization under a given loss function, is said to 
be noise-tolerant if the classifier it would learn with the noisy training data has the same 
probability of misclassification as that of the classifier the algorithm would learn if it is given 
ideal or noise-free class labels for all training data.  
Noise tolerance can be achieved even when $f_{\eta}^*\neq f^*$ because we are only comparing the 
probability of mis-classification of $f_{\eta}^*$ and $f^*$. However, $f^*_{\eta}=f^*$ is a sufficient condition for noise tolerance.

Thinking of an ideal noise-free sample allows us to properly formulate the noise-tolerance property as above.
We note once again that this noise-free sample is 
assumed to be unobservable. Making the probability of label corruption, $\eta_{\bf x}$, to be 
a function of ${\bf x}$ would take care of most cases of noisy data. 
For example, consider a 2-class problem with overlapping class conditional densities where the training 
data are generated by sampling from the respective class conditional densities. Then we can think of the 
unobservable noise-free dataset to be the one obtained by classifying the examples using Bayes optimal 
classifier. The labels given in the actual training dataset would not agree with the ideal labels (because 
of overlapping class conditional densities); however, the observed labels are easily seen to be 
noisy versions where the noise probability is a function of the feature vector. If there are any 
further sources of noise in generating the dataset given to the algorithm, these can also be easily 
accounted for by $\eta_{\bf x}$ because the probability of wrong label for different examples can be 
different. 

\section{Noise Tolerance of Risk Minimization}
In this section,
we analyze noise tolerance property of risk minimization with respect to different loss functions.
\subsection{0-1 Loss Function}
The 0-1 loss function is,
\begin{eqnarray}
L_{0-1}(f({\bf x}), y_{\bf x}) &=& I_{\{\mbox{sign}(f(\xx)) \neq y_{\xx}\}} \nonumber
\end{eqnarray}
where $I_{A}$ denotes indicator of event $A$. 

\begin{theorem}
Assume $\eta_{\xx} < 0.5, \; \forall \xx$. Then, 
(i). Risk minimization with 0-1 loss function is noise tolerant under uniform noise. (ii). In case of non-uniform noise, 
risk minimization with 0-1 loss function is noise tolerant if $R(f^*)=0$.
\label{thm-1}
\end{theorem}
\begin{proof}
For any $f$, let  $S(f):=\{{\bf x} \;|\;\mbox{sign}(f(\xx)) \neq y_{\bf x} \}$ and 
$S^c(f)=\{{\bf x} \;|\;\mbox{sign}(f(\xx)) = y_{\bf x} \}$. 
The risk for a function $f$ under no-noise case is 
\begin{eqnarray}\label{non-noisy}
 R(f) = E_{\xx} \Big{[} I_{\{\mbox{sign}(f(\xx)) \neq y_{\bf x}\}}\Big{]} =\int_{S(f)} \; dp({\bf x})\nonumber
\end{eqnarray}
where $dp({\bf x})$ denotes that the above integral is an expectation integral with respect to the distribution of feature vectors.
Recall $f^*$ is the minimizer of $R$.
Let $A(\xx):=I_{\{\mbox{sign}(f(\xx))\neq y_{\bf x}\}}$. Then risk for any $f$ in presence of noise would be  
\begin{eqnarray}
R^{\eta}(f) &=& E_{\bf x}\Big{[}(1-\eta_{\bf x}) A(\xx) + \eta_{\bf x} (1- A(\xx)) \Big{]} \nonumber \\
&=& \int_{\Re^d} \; \eta_{{\bf x}} dp({\bf x}) +\int_{S(f)} \; (1-2\eta_{\bf x}) dp({\bf x}) \label{agnostic_noisy}
\end{eqnarray}
 Given any $f \neq f^*$, using (\ref{agnostic_noisy}) we have
\begin{eqnarray}
R^{\eta}(f^*) - R^{\eta}(f) = \int_{S(f^*)\cap S^c(f)} \; (1-2\eta_{\bf x}) dp({\bf x}) \nonumber \\
  -\int_{S(f)\cap S^c(f^*)} \; (1-2\eta_{\bf x}) dp({\bf x}) \label{diff}
\end{eqnarray} 

For the first part of the theorem, we consider uniform noise and hence $\eta_{\bf x} = \eta,~\forall {\bf x}$. 
From Eq.(\ref{agnostic_noisy}),
we now get, for any $f$, $R^{\eta}(f)=\eta + (1-2\eta)R(f)$. Hence we have, $R^{\eta}(f)-R^{\eta}(f^*)
=(1-2\eta)(R(f)-R(f^*))$. 
Since $f^*$ is minimizer of $R$, we have $R(f)\geq R(f^*),\forall f\neq f^*$, 
which implies $R^{\eta}(f)\geq R^{\eta}(f^*),~\forall f\neq f^*$ if $\eta <0.5$. 
Thus under uniform label noise, $f^*$ also minimizes $R^{\eta}$. This completes proof of first part of the 
theorem. (The fact that risk minimization under 0--1 loss function is tolerant to uniform noise is 
known earlier (see, e.g. \cite[chap~4]{Sastrybook}))  

For the second part of the theorem, $\eta_{\bf x}$ is no longer constant but we assume $R(f^*) = 0$. 
This implies
$\int_{S(f^*)} \; dp({\bf x})=0$ and hence from Eq.(\ref{diff}), we get
$R^{\eta}(f^*)- R^{\eta}(f)\leq 0$ if $\eta_{\bf x} < .5,\forall {\bf x}$.
Thus, $f^*$, which is minimizer of $R$, also minimizes
$R^{\eta}$. This shows that risk minimization with 0-1 loss function is noise-tolerant under non-uniform noise 
if $R(f^*)=0$ and completes proof of the theorem. 
\end{proof}

If $R(f^*) \neq 0$, then risk minimization is, in general, not 
tolerant to non-uniform noise as we show by a counter example. 

\textbf{Example 1:} Fig. \ref{noisy} shows a binary classification problem where examples are generated using the 
true classifier $f^*_{\mbox{true}}(\xx)=\mbox{sign}(x_1^2+x_2)$.
Let the probability distribution on the feature space be uniformly concentrated on the training dataset. 
We note here that we get perfect classification if we consider quadratic classifiers. 
Since we want to consider the case where $R(f^*)>0$, we restrict the family of classifiers over which risk is minimized to linear classifiers.\\
(a) \textit{Without Noise: }The linear classifier which minimizes $R$ is
$f^*_{\mbox{lin}}(\xx)=x_2+5$ and $S(f^*_{\mbox{lin}})=\{\xx_9,\xx_{10}\}$.\\
(b) \textit{With Noise: }We now introduce non-uniform label noise in the data with the noise rates as follows: $\eta_{{\bf x}_9} =0.125$, $\eta_{{\bf x}_3} =0.4$, $\eta_{{\bf x}_5} =0.4$,
$\eta_{{\bf x}_7} =0.4$ and any noise rate (less than 0.5) to rest of the points. Consider another linear classifier 
$f^{\eta}_{\mbox{lin}}(\xx)=15.5x_1+8x_2+10$. From Fig. \ref{noisy}, 
we see that $S(f^{\eta}_{\mbox{lin}})=\{{\bf x}_3, {\bf x}_5, {\bf x}_7,{\bf x}_{10}\}$. Now using Eq.(\ref{diff}), we get
\begin{eqnarray}
&& R^{\eta} (f^*_{\mbox{lin}}) - R^{\eta} (f^{\eta}_{\mbox{lin}}) \nonumber \\
&= & \frac {(1-2\eta_{{\bf x}_9})-(1-2\eta_{{\bf x}_3})-(1-2\eta_{{\bf x}_5})-(1-2\eta_{{\bf x}_7})}{16} \nonumber \\
&=& \frac{(1-2*0.125)-3*(1-2*0.4)}{16} = \frac {0.15}{16} > 0. \nonumber
\end{eqnarray}
That is, $R^{\eta} (f^*_{\mbox{lin}}) > R^{\eta} (f^{\eta}_{\mbox{lin}})$, 
although $R(f^*_{\mbox{lin}}) < R(f^{\eta}_{\mbox{lin}})$.
\begin{figure}
\begin{center}
    \includegraphics[scale = .38]{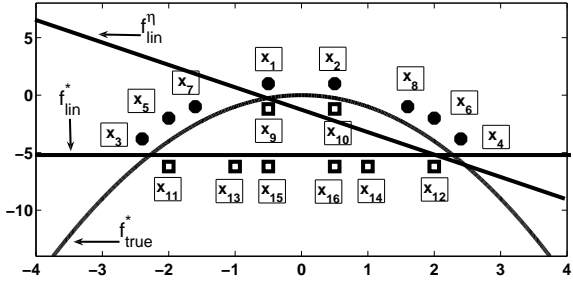}
\caption{Data for Example~1. Learning best linear classifier when the actual classification boundary is quadratic.}
\label{noisy}
\end{center}
\end{figure}

This example proves that risk minimization with 0-1 loss is, in general, not noise tolerant if $R(f^*) \neq 0$. 

{\bf Remark~1}: 
We note here that the assumption $R(f^*)=0$ may not be very restrictive; mainly because the noise free ideal data set is only a 
mathematical entity and need not be observable. For example, we can take $f^*$ to be the Bayes optimal classifier and assume that 
the ideal data set is obtained by classifying samples using the Bayes optimal classifier. Then $R(f^*)=0$. This means that if we 
minimize risk under 0-1 loss function with the actual training set, then the minimizer would be $f^*$. 

Finally, we note that all the above analysis is applicable to empirical risk minimization also by simply 
taking $p({\bf x})$ (in Eq.(\ref{non-noisy}) and (\ref{diff})) to be the empirical distribution.

While, as shown here, 0-1 loss function has impressive noise tolerant properties, 
risk minimization with this loss is difficult because it is a non-convex optimization problem. 
In machine learning, many other loss functions are used to make risk minimization
 computationally efficient. We will now examine the noise tolerance properties of other loss functions.  

\subsection{Squared Error Loss Function}
Squared error loss function is given by,
\begin{eqnarray}
\nonumber L_{\mbox{square}}(f({\bf x}), y_{\bf x}) = (f({\bf x})-y_{\bf x})^2 
\end{eqnarray}
We first consider the case when the function $f(\xx)$
is an affine function of $\xx$.
Let $f({\bf x})={\bf x}^T{\bf w}+b=\tilde{\bf w}^T\tilde{\bf x}$, where $\tilde{\bf w}=[{\bf w}~~b]^T \in \Re^{d+1}$ and
$\tilde{\bf x}=[{\bf x}~~1]^T \in \Re^{d+1}$.

\begin{theorem}
Risk minimization with squared error loss function for finding linear classifiers is noise tolerant 
under uniform noise if $\eta_{\xx} = \eta < 0.5$.
\label{thm-2}
\end{theorem}
\begin{proof}
For noise-free case, the risk is, $R(\tilde{\bf w}) = E\left[ (\tilde{\bf x}^T\tilde{\bf w} - y_{\bf x})^2 \right]$,
whose minimizer is $\tilde{\bf w}^*= \big{[}E[\tilde{\bf x}\tilde{\bf x}^T]\big{]}^{-1} E[\tilde{\bf x}y_{\bf x}]$. Risk under uniform label noise ($\eta_{\bf x} = \eta,\;\forall {\bf x}$) is given as 
\begin{eqnarray}\label{eq3}
R^{\eta}(\tilde{\bf w})
=& E_{\bf x} E_{\hat{y}_{\bf x}|{\bf x}}\left[(\tilde{\bf x}^T\tilde{\bf w}
- \hat{y}_{\bf x})^2 | {\bf x} \right] \nonumber \\
=& (1-\eta) E_{\bf x} \left[ (\tilde{\bf x}^T\tilde{\bf w} - y_{\bf x})^2 \right] +  \eta E_{\bf x} \left[ (\tilde{\bf x}^T\tilde{\bf w} + y_{\bf x})^2\right] \nonumber 
\end{eqnarray}
which is minimized by
\begin{eqnarray}
\tilde{\bf w}_{\eta}^*= (1-2\eta) \left[E_{\bf x}[\tilde{\bf x}\tilde{\bf x}^T]\right]^{-1} E_{\bf x}[\tilde{\bf x}y_{\bf x}] =(1-2 \eta) \tilde{\bf w}^*\nonumber
\end{eqnarray}
Since we assume $\eta < 0.5$, we have $(1-2 \eta) >0$. Hence 
we get, $\mbox{sign}(\bx^T\tilde{\bf w}_{\eta}^*)=\mbox{sign}(\bx^T\bw^*),~\forall \xx$. Which means
$P[\mbox{sign}(\bx^T\tilde{\bf w}_{\eta}^*)= y_{\xx}]=P[\mbox{sign}(\bx^T\tilde{\bf w}^*)= y_{\xx}]$. 
Thus under uniform noise, least square approach to learn linear classifiers is noise tolerant and the 
proof of theorem is complete.
\end{proof} 
\begin{corr}\label{corr1}
 Fisher Linear Discriminant (FLD) is noise tolerant under uniform label noise.
\end{corr}
\begin{proof}
For binary classification, FLD \cite{Bishop2006} finds direction ${\bf w}^*$ as, ${\bf w}^*={\arg\max}_{{\bf w}} \frac{{\bf w}^TS_B{\bf w}}{{\bf w}^TS_W {\bf w}}$, which is proportional to $S_W^{-1}(\muu_2-\muu_1)$. Here $S_B =(\muu_2-\muu_1)(\muu_2-\muu_1)^T$ and $S_W = \sum_{i=1}^2\sum_{{\bf x}_n\in C_i}({\bf x}_n-\muu_i)({\bf x}_n-\muu_i)^T$. $C_1$, $C_2$ represent the two classes and $\muu_1$, $\muu_2$ denote corresponding means. FLD can be obtained as the risk minimizer under squared error loss function \cite[chap~4]{Bishop2006} by choosing the target values as: $t_i=\frac{N}{N_1},~\forall {\bf x}_i\in C_1$ and
$t_i=-\frac{N}{N_2},~\forall {\bf x}_i\in C_2$, where $N_1=|C_1|$, $N_2=|C_2|$ and $N=N_1+N_2$.

When the training set is corrupted with uniform label noise, 
let $C^{\eta}_1$ and $C^{\eta}_2$ be the two sets now and
$\muu^{\eta}_1$ and $\muu^{\eta}_2$ corresponding means. Let
$N^{\eta}_1=|C^{\eta}_1|$ and $N^{\eta}_2=|C^{\eta}_2|$.
New target values are: $\hat{t}_i=\frac{N}{N^{\eta}_1},~\forall {\bf x}_i\in C_1^{\eta}$ and
$\hat{t}_i=-\frac{N}{N^{\eta}_2},~\forall {\bf x}_i\in C_2^{\eta}$.
The empirical risk in this case is,
$E^{\eta}(\ww,b)=\frac{1}{2}\sum_{i=1}^N ({\bf w}^T{\bf x}+b-\hat{t}_i)^2$.
Equating the derivative of $E^{\eta}$ with respect to $b$ to zero,
we get, $b=-{\bf w}^T\muu$,
where $\muu=\frac{1}{N}(N_1\muu_1+N_2\muu_2)$ is the mean of training set. Setting the gradient of $E^{\eta}$ with respect to $\ww$ to zero and using the values of $b$ and $\muu$, we get,
\begin{eqnarray}
&& \sum_{i=1}^N {\bf x}_i({\bf x}_i^T{\bf w}-\muu^T{\bf w}) = \sum_{i=1}^N \hat{t}_i{\bf x}_i \nonumber \\
&\Rightarrow & \Big{[}S_W+\frac{N_1N_2}{N}S_B \Big{]}{\bf w} = N(\muu^{\eta}_1-\muu^{\eta}_2) \nonumber\\
&\Rightarrow & \Big{[}S_W+\frac{N_1N_2}{N}S_B \Big{]}{\bf w} = N (1-2\eta)(\muu_1-\muu_2) \nonumber
\end{eqnarray}
where we have used the fact that, $\muu^{\eta}_1=(1-\eta)\muu_1+\eta \muu_2$ and $\muu^{\eta}_2=(1-\eta)\muu_2+\eta \muu_1$.
Note that $S_B{\bf w}\propto (\muu_2-\muu_1)$ for any ${\bf w}$. Thus we see that,
${\bf w}_{\eta}^* \propto S_W^{-1}(\muu_2-\muu_1)$.
Thus FLD is noise tolerant under uniform label noise.
\end{proof}
\textbf{Remark~2: } 
What we have shown is that risk minimization under squared error loss function is tolerant to uniform noise if we are learning linear classifiers. 
We can, in general, nonlinearly map feature vectors to a higher dimensional space so that the training set becomes linearly 
separable. Since uniform label noise in the original feature space should become uniform label noise in the 
transformed feature space, we feel that Theorem~2 should be true for risk minimization under squared error loss for any family of 
classifiers. 

Now consider the non-uniform noise case where  $\eta_{\bf x}$ is not same for all $\xx$.
Then, the risk $R^{\eta}$ is minimized by,
$\tilde{\bf w}_{\eta}^*= \big{[}E_{\bf x}[\tilde{\bf x}\tilde{\bf x}^T]\big{]}^{-1} E_{\bf x}[(1-2\eta_{\xx})\tilde{\bf x}y_{\bf x}]$.
Here, $\eta_{\xx}$ term can no longer be taken out of expectation. 
 Hence, we may not get noise tolerance. We show that it is so by a counter example as below. 
\begin{figure}[h]
 \begin{center}
  \includegraphics[scale=.35]{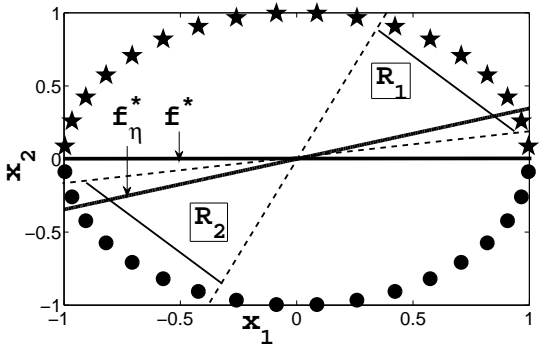}
\caption{Data for Example~2. $f^*$ is the classifier learnt when there is no noise. $f^*_{\eta}$ is the classifier learnt in presence of non-uniform label noise.}
\label{example4}
 \end{center}
\end{figure}

\textbf{Example 2: }Consider the unit circle centered at origin in $\Re^2$ and data points placed on its circumference as, $\xx_i=[\cos\theta_i~~\sin\theta_i]^T,\;\theta_i=\frac{(2i-1)\pi}{36},\;i=1\ldots 36$. $y_{\xx_i}=1,\;i=1\ldots 18$ and $y_{\xx_{i+18}}=-1,\;i=1\ldots 18$.
Assume that the probability distribution on the feature space is uniformly concentrated on the training dataset. Let the set of classifiers contain only linear classifiers passing through origin.\\
(a) \textit{Without Noise: }In this case, risk is minimized by $\ww^*=[0~~1.27]^T$. Classifier, $\mbox{sign}(\xx^T\ww^*)$, linearly separates the two classes. Thus $P[\mbox{sign}(\xx^T\ww^*)=y_{\xx}]=1$.\\
(b) \textit{With Noise: }Now let us introduce non-uniform label noise as follows.
$\eta_{\xx_i}=0.4,\;\xx_i\in R_1\cup R_2$, where $R_1=\{\xx_2,\ldots,\xx_7\}$ and $R_2=\{\xx_{20},\ldots,\xx_{25}\}$; 
 $\eta_{\xx_i}=0$ for rest of the points.
In this case, risk is minimized by $\ww^*_{\eta}=[-0.342~~0.988]^T$.
$\mbox{sign}(\xx^T\ww^*_{\eta})$ mis-classifies $\xx_1,\xx_2,\xx_{19}$ and $\xx_{20}$
as shown in Fig. \ref{example4}.
Hence $P[\mbox{sign}(\xx^T\ww_{\eta}^*)= y_{\xx}]=\frac{8}{9} \neq 1$.  

Thus squared error loss is not noise tolerant under non-uniform noise even if the minimum risk under 
noise-free case is zero and the optimal classifier is linear in parameters.\\
\textbf{Remark~3: }An interesting special case of non-uniform noise is class conditional classification noise (CCCN) \cite{StempfelR07a}, 
 where $\eta_{\xx}=\eta_i$ for $\xx \in C_i, \: i=1,2$. 
Least squares method may not be tolerant to such a non-uniform noise. However, using the proof of Corollary~\ref{corr1}, it is easy to verify
that FLD is noise tolerant under CCCN.

\subsection{Exponential Loss Function}
Exponential loss function is given by,
\begin{eqnarray}
L_{\exp}(f({\bf x}), y_{\bf x}) &=& \exp(-y_{\bf x}f({\bf x}))\nonumber
\end{eqnarray}
This is the effective loss function for adaboost. We show, through the following counter example, that exponential loss function is not tolerant to even uniform  noise.

\textbf{Example 3: }
Let $\{(x_1,y_{x_1}), (x_2, y_{x_2}),(x_3,y_{x_3})\}$ be the training dataset such that
$x_1=5$, $x_2=10$ and $x_3=11$, with $y_{x_1}=-1$, $y_{x_2}=-1$ and $y_{x_3}=+1$. 
Let the probability distribution on the feature space be uniformly concentrated on the training dataset.
Here, we find a linear classifier which minimizes the risk under exponential loss function. We consider linear classifiers 
 expressed as $\mbox{sign}(x+b)$.\\
(a) \textit{Without Noise: }The risk of a linear
classifier without label noise is written as:
\begin{eqnarray}
 R(b)=\frac{1}{3}\Big{[}e^{5+b}+e^{10+b}+e^{-11-b}\Big{]} \nonumber
\end{eqnarray}
By equating the derivative of $R(b)$ to zero, we get,
\begin{eqnarray}
\nonumber  && e^{5+b}+e^{10+b}-e^{-11-b}=0 \\
\nonumber \Rightarrow && b^*=\frac{1}{2}\ln\Big{(}\frac{e^{-11}}{e^{5}+e^{10}}\Big{)}=-10.5034
\end{eqnarray}
$\mbox{sign}(f(x))=\mbox{sign}(x+b^*)$ correctly classifies all the points.
Thus $P[\mbox{sign}(x+b^*)=y_x]=1$.\\
(b) \textit{With Noise: }Now let us introduce uniform label noise with noise rate $\eta=0.3$.
The risk will be 
\begin{eqnarray}
 \nonumber R^{\eta}(b)&=&\frac{(1-\eta)}{3}\Big{[}(e^{5+b}+e^{10+b}+e^{-11-b})\Big{]}\\
\nonumber &&+ \frac{\eta}{3}\Big{[}(e^{-5-b}+e^{-10-b}+e^{11+b})\Big{]}
\end{eqnarray}
Again equating the derivative of $R^{\eta}(b)$ to zero, we get,
\begin{eqnarray}
\nonumber & &  (1-\eta)\big{(}e^{5+b}+e^{10+b}-e^{-11-b}\big{)}-\eta\big{(}e^{-5-b}\\
\nonumber && +e^{-10-b}-e^{11+b}\big{)}=0 \\
\nonumber & \Rightarrow & b^*_{\eta}=\frac{1}{2}\ln \Big{(}\frac{0.7e^{-11}+0.3(e^{-5}+e^{-10})}{0.7(e^{5}+e^{10})+0.3 e^{11}}\Big{)}= -8.3052
\end{eqnarray}
$\mbox{sign}(f(x))=\mbox{sign}(x+b^*_{\eta})$ mis-classifies $x_2$.
Thus $P[\mbox{sign}(x+b_{\eta}^*)= y_x]=\frac{2}{3} \neq P[\mbox{sign}(x+b^*)= y_x] $.
Thus risk minimization under exponential loss is not noise tolerant even with uniform noise. 

\subsection{Log Loss Function}
Log loss function is given by,
\begin{eqnarray}
L_{\log}(f({\bf x}), y_{\bf x}) &=& \ln(1+\exp(-y_{\bf x}f({\bf x})))\nonumber
\end{eqnarray}
This is the effective loss function for logistic regression. Risk minimization with 
log loss function also is not noise tolerant. We demonstrate it using following counter example.

\textbf{Example 4: }
Consider the same training dataset as in Example 3. 
We need to find a linear classifier, $\mbox{sign}(x+b)$,  which minimizes the risk under log loss function. \\
(a) \textit{Without Noise: }
The risk of a linear classifier without label noise is
\begin{eqnarray}
 \nonumber R(b)=\frac{\ln(1+e^{5+b})+\ln(1+e^{10+b})+\ln(1+e^{-11-b})}{3}.
\end{eqnarray} 
Equating the derivative of $R(b)$ to  zero, we get,
\begin{eqnarray}
\nonumber & & \frac{e^{5+b}}{1+e^{5+b}}+\frac{e^{10+b}}{1+e^{10+b}}-\frac{e^{-11-b}}{1+e^{-11-b}}=0 \\
\nonumber &\Rightarrow & 2e^{26}t^3+(e^{15}+e^{21}+e^{16})t^2-1=0,
\end{eqnarray}
where $t = e^b$. Roots of this polynomial are $-0.0034$, $-2.75\times 10^{-5}$ and $2.73\times 10^{-5}$. 
The only positive root is $t=2.73\times 10^{-5}$. Using this value of $t$, we get $b^*=\ln(t)=-10.5086$.
$f(x)=x+b^*$ classifies all the points correctly.
Thus $P[\mbox{sign}(x+b^*)=y_x]=1$.\\
(b) \textit{With Noise: }Now let us introduce uniform label noise with noise rate $\eta=0.3$.
The risk will be, 
\begin{eqnarray}
 \nonumber R^{\eta}(b)&=&\frac{(1-\eta)}{3}\Big{[}\ln(1+e^{5+b})+\ln(1+e^{10+b})+ \\
\nonumber && \ln(1+e^{-11-b})\Big{]}+ \frac{\eta}{3}\Big{[}\ln(1+e^{-5-b})\\
\nonumber && +\ln(1+e^{-10-b})+\ln(1+e^{11+b})\Big{]} 
\end{eqnarray}
Equating the derivative of $R^{\eta}(b)$ to zero, we get a sixth degree polynomial in $t=e^b$ which has only one positive root. 
This root gives us the value of $b^*_{\eta}=-9.8607$. 
The classifier, $\mbox{sign}(f(x))=\mbox{sign}(x+b^*_{\eta})$ mis-classifies $x_2$.
Which means $P[\mbox{sign}(x+b_{\eta}^*)= y_x]=\frac{2}{3}$.

Thus, $P[\mbox{sign}(x+b_{\eta}^*)= y_x]\neq P[\mbox{sign}(x+b^*)= y_x]$ and 
log loss is not noise tolerant even with uniform noise.

\subsection{Hinge Loss Function}
This is a convex loss function and has the following form.
\begin{eqnarray}
 \nonumber L_{\mbox{hinge}}(f({\bf x}),y_{\bf x})=\max(0,1-y_{\bf x}f({\bf x})) 
\end{eqnarray}
Support vector machine is based on minimizing risk under the hinge loss. Here we show that
hinge loss function is not noise tolerant using a counter example.

\textbf{Example: 5 }Consider the same training dataset as in Example~3. 
Here we consider learning linear classifiers expressed as $\mbox{sign}(wx+b)$.\\
(a) \textit{Without Noise: }The risk of a linear classifier with noise-free training data is 
\begin{eqnarray}
 \nonumber R(w,b)&=&\frac{1}{3}\sum_{n=1}^3 \max[0,1-y_{x_n}(wx_n+b)]
\end{eqnarray}
To find the minimizer of $R(w,b)$, we need to solve 
\begin{eqnarray}
\nonumber &         \min_{w,b,\xi_1,\xi_2,\xi_3}& \frac{1}{3}\sum_{n=1}^3\xi_n\\
\nonumber &  s.t. & 5w+b\leq -1+\xi_1,\;\;\xi_1\geq 0                         \\
\nonumber &       & 10w+b\leq -1+\xi_2,\;\;\xi_2\geq 0                        \\
\nonumber &       & 11w+b\geq 1-\xi_3,\;\;\xi_3\geq 0
\end{eqnarray}
The optimal solution of the above linear program is $(w^*,b^*)=(54.7738,-571.221)$
which is also the minimizer of $R(w,b)$.  
$\mbox{sign}(w^*x+b^*)$ classifies all the points correctly. 
Thus $P[\mbox{sign}(w^*x+b^*)=y_{x}]=1$.\\
(b) \textit{With Noise: }
Now we introduce uniform label noise with noise rate $\eta=0.3$ in the training data. 
The risk of a linear classifier in presence of uniform label noise is 
\begin{eqnarray}
 \nonumber R^{\eta}(w,b)&=&\frac{1}{3}\sum_{n=1}^3 \big{[}(1-\eta)\max[0,1-y_{x_n}(wx_n+b)]\\
\nonumber && +\eta \max[0,1+y_{x_n}(wx_n+b)\big{]}
\end{eqnarray}
Minimizing of $R^{\eta}(w,b)$ by solving the equivalent linear program as earlier, we get 
$(w^*_{\eta},b_{\eta}^*)=(0.3333,-2.6667)$.
 The classifier $\mbox{sign}(w^*_{\eta}x+b^*_{\eta})$ mis-classifies $x_2$.
Thus $P[\mbox{sign}(w_{\eta}^*x+b^*_{\eta})= y_{x}]= \frac{2}{3} \neq P[\mbox{sign}(w^*x+b^*)=y_{x}]$.
Thus hinge loss is not noise tolerant even under uniform noise even when the optimal classifier is linear.

\section{Some Empirical Results}
In this section, we present some empirical evidence for our theoretical results. The main difficulty 
in doing such simulations is that there is no general purpose algorithm for risk minimization under 
0--1 loss.  Here we use the CALA-team algorithm proposed in \cite{Sastry09} which (under sufficiently 
small learning step-size) converges to  minimizer of risk under 0--1 loss in case of linear classifiers. 
Hence, here we restrict the simulations only to learning of linear classifiers and hence 
 give experimental results on Iris dataset.

Iris recognition is a three class classification problem in 4-dimensions.   
The first class, Iris-setosa, is linearly separable from the other two classes, namely, Iris-versicolor and Iris-virginica.  
We consider a linearly separable 2-class problem by combining the latter two classes as one class.

The original Iris data set 
has no label noise. We introduce different rates of
uniform noise varying from 10\% to 30\%.
We incorporated non-uniform label noise as follows. For every example, the probability of flipping
the label is based on which quadrant (with respect to the first two
features) the example falls in. 
The noise rate in this case is represented
by a quadruple with $i$-th element representing probability
of wrong  class label if the feature vector is in $i^{th}$ quadrant ($i=1,2,3,4$).

For training, we use entire dataset with label noise inserted to it.
We use the original noise-free examples for testing.
We use test error rate as an indicator of the noise-tolerance. We compare 
CALA algorithm for risk minimization under 0--1 loss with 
SVM (hinge loss), linear least square (squared error loss),
and logistic regression (log loss) which are risk minimization algorithms under different convex loss functions. 

The results are shown in Table~\ref{tab6}. For each noise rate, we generated ten random 
noisy training data sets. We show the mean and standard deviation of accuracy on test set with 
each of the algorithms. (The CALA algorithm \cite{Sastry09} is a stochastic one and hence has a 
non-zero standard deviation even in the case of no-noise data). As can be seen from the table, risk minimization under 
0--1 loss has impressive noise tolerance under both uniform and non-uniform label noise. Both SVM and logistic 
regression have the highest accuracy under no-noise; but their accuracy drops from 98\% to 89\% and 91\% respectively 
under uniform noise rate of 20\%. Linear least squares algorithm achieves accuracy of 92\% when there is no noise and it 
drops to only 91\% when 20\% uniform noise is added, showing that it is tolerant to uniform noise. 
(The performance of Fisher linear discriminant is similar to that of linear least squares: it achieves accuracy of 
94\%, 92.20\%$\pm$2.49, 91.07\%$\pm$2.88, 90.27\%$\pm$2.16 respectively on 0\%, 10\%, 20\% and 30\% uniform noise 
and 91.53\%$\pm$1.72 and 87.67\%$\pm$2.71 on the two cases of non-uniform noise). 
Also, the large standard deviations 
of SVM and other algorithms in the non-uniform noise case  show their 
 sensitivity to noise.
\begin{table}
\begin{center}
  \begin{tabular}{|p{.62in}|p{.49in}|p{.57in}|p{.5in}|p{.5in}|}
\hline 
              &  0--1 loss         &   hinge loss   &  sq. err. loss    & log loss  \\
 Noise Rate   &  (CALA)              &    (SVM,C=$10^3$)&  (Least Sq.)  & (LogReg)          \\ \hline 
No Noise      &  97.53$\pm$0.38    &    98.67       &  92.67         &  98.67          \\ \hline
Uniform 10\%  &  97.47$\pm$0.98    & 93.40$\pm$2.92 &  92.53$\pm$1.33&  92.87$\pm$1.47 \\ \hline
Uniform 20\%  &  97.07$\pm$1.09    & 89.47$\pm$4.02 &  91.47$\pm$1.17&  91.67$\pm$1.87 \\ \hline
Uniform 30\%  &  97.07$\pm$1.05    & 83.73$\pm$6.79 &  90.13$\pm$1.77&  90.07$\pm$1.99 \\ \hline
Non-Uniform 15,20,25,30\% & 96.47$\pm$1.49 & 89.67$\pm$3.18 & 91.27$\pm$1.49 & 91.67$\pm$2.07 
\\ \hline
Non-Uniform 30,25,20,15\% & 97.00$\pm$1.01 &  82.47$\pm$7.04 &   85.80$\pm$5.07 & 85.93$\pm$5.09 \\ \hline
\end{tabular} 
\end{center} 
\caption{Simulations results with Iris data} 
\label{tab6}
\end{table}

\section{Conclusion}
While learning a classifier, one has to often contend with noisy training data. In this paper, we presented some analysis to bring out
the inherent noise tolerant properties of the risk minimization strategy under different loss functions.

Of all the loss functions, the 0-1 loss function has best noise tolerant properties. 
We showed that it is noise tolerant under uniform noise and also under non-uniform noise if the risk minimizer achieves 
zero risk on uncorrupted or noise-free data.

If we consider the case where we think of our ideal noise-free sample as the one obtained by 
classifying {\em iid} feature vectors using Bayes optimal classifier, the minimum risk achieved 
would be zero if the family of classifiers over which the risk is minimized includes the 
structure of Bayes classifier. In such a case, the noise-tolerance (under non-uniform label 
noise) of risk minimization implies that if we find the classifier to minimize risk under 
0-1 loss function (treating the labels given in our training data as correct), we would (in a probabilistic sense) automatically 
learn the Bayes optimal classifier. This is an interesting result that makes risk minimization under  
0-1 loss a very attractive classifier learning strategy. 

A problem with minimizing risk under 0-1 loss function is that it is difficult to use any standard 
optimization technique to minimize risk due to discontinuity of loss function. Hence, given the 
noise-tolerance properties presented here, an interesting problem to address is that of some gradient-free 
optimization techniques to minimize risk under 0-1 loss function.  
For the linear classifier case, the stochastic optimization algorithm 
proposed in \cite{Sastry09} is one such algorithm. To really exploit the 
noise-tolerant property of the 0-1 loss function we need more efficient techniques of that kind and also techniques which 
work for nonlinear classifiers. 

On the other hand, risk under convex loss functions 
is easy to optimize. 
Many generic classifiers are based on minimizing risk under these convex loss function.
But it is observed in practice that in presence of noise, these approaches over-fit.

In this paper, we showed that these convex loss functions are not noise tolerant. 
Risk minimization under hinge loss, exponential loss and log loss
is not noise tolerant even under uniform label noise. This explains the problem one faces with algorithms 
such as SVM if the class labels given are sometimes incorrect.
We also showed that the linear least
squares approach is noise tolerant under uniform noise but
not under non-uniform noise. Same is shown to be true for
Fisher linear discriminant.  

Most algorithms for learning classifiers focus on minimizing risk under a convex loss function
to make the optimization  more
tractable. The analysis 
presented in this paper suggests that looking for
 techniques to minimize risk under 0-1 loss function may be a promising approach  
for classifier design especially when we have to learn from noisy
training data.
\bibliographystyle{ieeetr}
\bibliography{noise}
\end{document}